\newif\iffull
\fulltrue

\iffull
\documentclass[11pt,twoside]{article}
\usepackage{fullpage,amsmath,amsthm}
\else
\documentclass[anon,12pt]{colt2018}
\usepackage{times}
\fi

\usepackage{amssymb,color,hyperref}
\usepackage{vfmacros,framed,algorithm,algorithmic}
\newcommand{\INDSTATE}[1][1]{\STATE\hspace{#1\algorithmicindent}}

\iffull
\usepackage[style=alphabetic,backend=bibtex,maxbibnames=20,maxcitenames=6,giveninits=true,doi=false,url=false]{biblatex}
\newcommand*{\citet}[1]{\AtNextCite{\AtEachCitekey{\defcounter{maxnames}{2}}} \textcite{#1}}

\newcommand*{\citep}[1]{\cite{#1}}

\bibliography{vf-allrefs-central,dp-api}
\fi

\newif\ifnotes
\notestrue
\ifnotes
        \newcommand{\mnote}[1]{{\sf \textcolor{blue}{#1}}}
\else
		\newcommand{\mnote}[1]{}
\fi

\providecommand{\err}{\mathsf{Err}}
\providecommand{\opt}{\mathsf{Opt}}
\providecommand{\K}{{\mathcal K}}

\providecommand{\D}{{\mathcal D}}
\providecommand{\cE}{{\mathcal E}}

\providecommand{\cP}{{\mathcal P}}

\newcommand{\Thr}{\mathsf{Thr}}
\newcommand{\proj}{\mathsf{Proj}}
\newcommand{\epw}{\mathrm{ExpPW}}

\begin{document}

\title{Privacy-preserving Prediction\footnotetext{Accepted for presentation at Conference on
Learning Theory (COLT) 2018.}}
\iffull
\author{Cynthia Dwork \\ Harvard University \and Vitaly Feldman \\ Google Brain}
\fi

\date{}
\maketitle

\begin{abstract}
Ensuring differential privacy of models learned from sensitive user data is an important goal that has been studied extensively in recent years.
It is now known that for some basic learning problems, especially those involving high-dimensional data, producing an accurate private model requires much more data than learning without privacy. At the same time, in many applications it is not necessary to expose the model itself. Instead users may be allowed to query the prediction model on their inputs only through an appropriate interface. Here we formulate the problem of ensuring privacy of individual predictions and investigate the overheads required to achieve it in several standard models of classification and regression.

We first describe a simple baseline approach based on training several models on disjoint subsets of data and using standard private aggregation techniques to predict. We show that this approach has nearly optimal sample complexity for (realizable) PAC learning of any class of Boolean functions. At the same time, without strong assumptions on the data distribution, the aggregation step introduces a substantial overhead.
We demonstrate that this overhead can be avoided for the well-studied class of thresholds on a line and for a number of standard settings of convex regression. The analysis of our algorithm for learning thresholds relies crucially on strong generalization guarantees that we establish for all differentially private prediction algorithms.
\end{abstract}

\section{Introduction and problem formulation}
In machine learning tasks, the training data often consists of information collected
from individuals. This data can be highly sensitive, for example in the case
of medical or financial information, and therefore privacy-preserving data analysis is becoming an increasingly important area of study in machine learning, data mining and statistics \citep{DworkSmith:09,SarwateC:13,DworkRoth:14}. We rely on the well-studied differential privacy model of privacy that has become a de facto standard for formal understanding of privacy \citep{DworkMNS:06}.

The standard setting of privacy-preserving learning aims to ensure that the model learned from the data is produced in a differently private way. Thus this approach preserves privacy even when a potential adversary has complete access to the description of the predictive model. The downside of this strong guarantee is that for some learning problems, achieving the guarantee is known to have substantial additional costs. More examples are needed to achieve the same level of accuracy (or lower accuracy is achievable for a given number of examples). In addition, private learning may require new and computationally less efficient algorithms.

In this work we consider learning in a setting where the description of the learned model is not accessible to the (potentially adversarial) user(s). Instead the users have access to the model through an interface (often referred to as an API). For an input point the interface provides the value of the predictive model on that point. This view is appropriate for many existing applications where user privacy is a concern. For example, companies that collect data about their users usually expose only a cloud-based interface to the models they train on user data. Credit rating bureaus only allow access to their models through an electronic interface. In addition, it may enable new applications where privacy considerations are currently preventing the use of predictive models trained on sensitive user data. For example, in medical diagnostics a prediction interface would suffice for most applications.

Allowing such restricted access may appear to pose no risk to individual privacy. However, as recently demonstrated by \citet{ShokriSSS17}, blackbox access to Amazon ML and Google prediction APIs suffice for successful membership inference attacks. Membership inference is the task in which given a user's record the goal is to infer whether the record was used for training the model. This information is known to be sensitive in several contexts. Membership inference can also be used to complete partial records revealing the values of sensitive attributes. Even more recently, \citet{LongBWBW18} demonstrated several additional successful membership inference attacks based on blackbox access. Further, \citet{CarliniLKES18} proposed a more formal way to measure the degree to which sensitive information is memorized by generative sequence models and explored several techniques to extract sensitive information using black box access to such models. The use of differentially private learning algorithms to protect against such attacks has been proposed in \citep{ShokriSSS17} and briefly explored in \citep{CarliniLKES18}.

We now describe the setting more formally. For a prediction problem over a domain $X$ and label space $Y$, a prediction interface is an algorithm that has access to a dataset $S \in (X\times Y)^n$ and given a query point $x\in X$ outputs a value $y \in Y$. The algorithm can be queried multiple times and is stateful (namely, responses can depend on previous queries). We define the privacy of such an interface in the same way as usually done for interactive algorithms. Namely, for a prediction interface $M$ and a stateful query generating algorithm $Q$ we denote by $(Q\rightleftarrows M(S))$ the sequence of queries and responses generated in the interaction of $Q$ and $M$ on dataset $S$.
\begin{defn}[Private prediction interface]
A prediction interface $M$, is $(\eps,\delta)$-differentially private if for every interactive query generating algorithm $Q$, the output $(Q\rightleftarrows M(S))$ is $(\eps,\delta)$-differentially private with respect to dataset $S$.
\end{defn}

While the problem setting has many facets that merit investigation, we focus on perhaps the most basic question: what is the cost of ensuring privacy of a single prediction. In other words, we focus on the problem of answering a single prediction query. Composition properties of differential privacy imply that such an algorithm can be used to answer multiple queries with privacy parameters that degrade gracefully with the number of queries \citep{DworkRoth:14}. Therefore such an algorithm is a natural building block for constructing an algorithm that can answer multiple queries. Naturally, better ways of dealing with sequences of queries might exist and the general topic of answering interactive sequences of queries has been studied extensively in the differential privacy literature (see \citep{DworkRoth:14} for an overview).

An algorithm $M$ that answers a single query $x$ defines a randomized prediction at $x$ and hence such an algorithm implicitly defines a learning algorithm that outputs a randomized predictor $h(x) = M(S,x)$.
\begin{defn}
\label{def:prediction-privacy}
Let $M$ be an algorithm that given a dataset $S\in (X\times Y)^n$ and a point $x$ produces a value in $Y$. We say that $M$ is {\em  $(\eps,\delta)$-differentially private prediction} algorithm if for every $x \in X$, the output $M(S,x)$ is $(\eps,\delta)$-differentially private with respect to $S$. We use $M(S)$ to refer to the (randomized) function $M(S,\cdot)$.
\end{defn}

This definition allows us to treat this building block in the same way as regular learning algorithms and discuss it in the context of standard statistical learning models.

Two standard and closely related models for classification we will look at are PAC (or realizable) learning  \citep{Valiant:84} and agnostic \citep{Haussler:92,KearnsSS:94} learning.  In the PAC learning model the algorithm is given random examples in which each point is sampled i.i.d.~from some unknown distribution over the domain and is labeled by an unknown function from a set of functions $C$. In the agnostic learning model the algorithm is given examples sampled i.i.d.~from an arbitrary (and unknown) distribution over labeled points. The goal of the learning algorithm in both models is to output a hypothesis whose prediction error on the distribution from which examples are sampled is within additive $\alpha$ of the prediction error of the best function in $C$ (which is $0$ in the PAC model). See Sec.~\ref{sec:prelims} for formal definitions.

We will also consider a more general regression setting in which we are given a loss function $\ell: \R\times Y \rightarrow \R$ and the goal is to design a private prediction algorithm $M$ that minimizes  $$\cE_{\cP}[\ell(M(S)] = \E_{M, (x,y) \sim \cP}[\ell(M(S,x),y)] ,$$
where $\cP$ is an unknown probability distribution  over $X\times Y$.


\section{Overview of the results}
We first consider a natural ``baseline'' approach to this problem based on private aggregation of non-private learning algorithms.

\subsection{Private aggregation of non-private models}
To produce a prediction differentially privately we partition the dataset $S$ into several subsamples $S_1,\ldots,S_r$ and run a non-private learning algorithm on each of those subsamples too obtain predictors $f_1,\ldots,f_r$. Now given a point $x$ we use a differentially private aggregation technique on values $f_1(x),\ldots,f_r(x)$ and output the result. Several such subsample-and-aggregate techniques are known \citep{NissimRS07,DworkLei:09,SmithT13,DworkRoth:14} that carefully exploit properties of the distribution over results on subsamples. A significant advantage of this approach is that it does not require a new learning algorithm and hence is easy to implement (there is an additional computational cost that is easy to parallelize).

Obviously, using $r$ subsamples requires more data than non-private learning and therefore it is natural to ask whether this approach is optimal and how it compares to differentially private learning in the standard setting. We discuss these questions in the context of specific problems below.

\paragraph{PAC Learning:}
For PAC learning (or realizable case) accurate models $f_1,\ldots,f_r$ have to be close to the true labeling function $f$ (that is, they disagree with probability at most $\alpha$).
In particular, the fraction of points on which more than $1/4$ of the predictors output the wrong label cannot be more than $4\alpha$. Outputting the correct label with privacy is easy in this setting and we do this using a soft majority vote (or, equivalently, the exponential mechanism \citep{McSherryTalwar:07} on the label counts). A number of other approaches would give comparable guarantees. A simple analysis shows that using $r = O(\ln(1/\alpha)/\eps)$ this reduction ensures $\eps$-differentially private prediction (see Thm.~\ref{thm:pac-reduction} for a formal statement).

As an immediate corollary of this reduction and standard bounds on the sample complexity of PAC learning we obtain the following upper bound.
\begin{cor}
\label{cor:pac}
Let $C$ be a class of Boolean functions of VC dimension $d$. Then for all $\alpha,\beta,\eps > 0$, there exists an $\eps$-differentially private prediction algorithm $M$ that PAC learns $C$ with error $\alpha$ and confidence $1-\beta$ given $n = \tilde O\lp\frac{d + \log(1/\beta)}{\eps\alpha}\rp$ examples.
\end{cor}

It turns out that this simple approach is essentially optimal in the worst case. Specifically, we prove that the sample complexity of this problem is $\Omega(d/(\eps\alpha))$ even when $\delta$ is as large as $\eps/3$.
\begin{thm}
\label{thm:pac-lower-bound}
Let $C$ be a class of Boolean functions of VC dimension $d$. Then for all $\alpha,\eps > 0$, any $(\eps,\eps/3)$-differentially private prediction algorithm $M$ that PAC learns $C$ with error $\alpha$ and confidence $1/12$ requires $n=\Omega(d/(\eps\alpha))$ examples.
\end{thm}

For comparison, \citet{KasiviswanathanLNRS11} showed that the sample complexity of differentially privately PAC learning a class $C$ over domain $X$ is $O(\log(|C|)/(\eps\alpha))$. By Sauer's lemma, $\log(|C|) = O(d \cdot \log(|X|))$ and therefore the multiplicative gap between these two measures can be as large as $\log(|X|)$. The sample complexity of $\eps$-differentially private PAC learning was subsequently shown to be $\tilde{\Theta}(R/(\eps\alpha))$, where $R$ is the so-called representation dimension of $C$ \citep{BeimelNS:13}. However, as shown in \citep{FeldmanXiao15}, for many classes the gap between $R$ and the VC dimension is still roughly $\log(|X|)$. For example, the representation dimension of linear threshold functions over $[N]^p$ is $p^2 \cdot \log N$ whereas the VC dimension is just $p$.

We remark that the technique we use to prove the lower bound in Thm.~\ref{thm:pac-lower-bound} is different from those used for proving lower bounds in the standard setting of learning with privacy.

\paragraph{Agnostic learning:}
In agnostic learning, the labels $f_1(x),\ldots,f_r(x)$ no do not necessarily agree on most points $x$ and taking the majority vote may even reduce the accuracy. In this setting we predict by first averaging the non-private predictions to obtain $v(x)=\fr{r}(f_1(x)+\cdots+f_r(x))$ and then outputting $1$ with probability $v(x) + \zeta$ (truncated to range $[0,1]$), where $\zeta$ is a Laplace noise variable. It is not hard to show that for $r=O(1/(\eps\alpha))$, this approach ensures that the prediction will be $\eps$-differentially private and the addition of noise increases the prediction error by at most an extra $\alpha$ term (see Cor.~\ref{cor:agnostic-reduction}). As a corollary of this reduction, we obtain the following upper-bound on the sample complexity in this setting.
\begin{cor}
\label{cor:agnostic-sample-complexity-intro}
Let $C$ be a class of Boolean functions of VC dimension $d$. Then for all $\alpha,\beta,\eps > 0$ there exists an $\eps$-differentially private prediction algorithm $M$ that agnostically learns $C$ with excess error $\alpha$ and confidence $1-\beta$ given $n = \tilde O\lp\frac{d  + \log(1/\beta)}{\eps\alpha^3}\rp$ examples.
\end{cor}
In this case the upper bound is much worse than the lower bound of $\Omega(d/\alpha^2 + d/(\eps\alpha))$ implied by Thm.~\ref{thm:pac-lower-bound}. For comparison, $\eps$-differentially private agnostic learning can be done using $\tilde O(d/\alpha^2 + R/(\eps\alpha))$ examples, where $R$ is the representation dimension of $C$ mentioned above \citep{BeimelNS:13,FeldmanXiao15}. As a result, for classes such that $R=O(d)$ a differentially private learning algorithm matches the lower bound for private prediction. This leads to a natural question of whether it is possible to match the lower bound for all classes $C$. While we do not answer this question for arbitrary classes $C$, we give an example of an algorithm that goes beyond these two approaches. Specifically, it agnostically learns $C$ with $\eps$-private prediction using $\tilde O(d/\alpha^2 + d/(\eps\alpha))$ examples whereas learning $C$ with privacy in the standard model requires an infinite number of examples.

\paragraph{Convex regression:}
Our analysis of agnostic learning can be seen as a special case of a more general analysis of prediction problems with convex loss functions. Specifically, the aggregation by averaging can be seen as a way to increase the {\em uniform prediction stability} of a learning algorithm. A learning algorithm is uniformly prediction stable with rate $\gamma$ if for predictors $f_S$ and $f_{S'}$ produced on any pair of datasets $S,S'$ that differ on a single element and any point $x$, $|f_S(x) - f_{S'}(x)| \leq \gamma$. As follows immediately from this definition, a uniformly prediction stable learning algorithm can be converted to a differentially private prediction algorithm simply by adding Laplace (or Gaussian) noise to the prediction (see Lem.~\ref{lem:stability-to-privacy}). Hence it reduces our problem to the problem of finding a uniformly prediction stable learning algorithm with sufficiently low rate of stability. Aggregation by averaging the predictors obtained by running a learning algorithm on $r$ disjoint datasets can be seen as improving its uniform prediction stability by a factor of $r$. Convexity of the loss function, in turn, ensures that such averaging preserves the guarantees on the expected loss of the algorithm (see Lem.~\ref{lem:amplify-stability} for a formal statement).

We demonstrate how this general approach can be applied to convex regression problems.  Specifically, we consider problems in which we have a family of predictors $\{ f(w,\cdot)\}_{w\in \K}$ parameterized by a vector $w \in \K $, where $\K \subset \R^d$ is some convex body, $\ell$ is a convex loss function and $\ell(f(\cdot,x),y)$ is a convex function of $w$ over $\K$ for all $(x,y) \in X\times Y$. The goal is to find $\hat{w}$ such that $$\E_{(x,y)\sim\cP}[\ell(f(\hat{w},x),y)] \leq \min_{w \in \K} \E_{(x,y)\sim\cP}[\ell(f(w,x),y)]+\alpha,$$  where $\cP$ is an unknown distribution over examples. This setting captures many important learning problems and has also been extensively investigated in the privacy literature (see \citep{ChaudhuriMS11,KiferST12,BassilyST14,TalwarTZ15,WangYX17} and references therein). For the purpose of comparison with sample complexity bounds known in this literature we restrict our attention to a basic setting in which $\K$ is a subset of the unit Euclidean ball and $\ell(f(w,x),y)$ is 1-Lipschitz in $w$ for all $(x,y)$ in support of $\cP$. For this setting it is known that $\tilde{O}(d/(\eps \alpha^2))$ samples suffice to solve the problem with $\eps$-differential privacy and $\tilde{O}(\sqrt{d}\log^4(1/\delta)/(\eps \alpha^2))$ samples suffice for $(\eps,\delta)$-differential privacy \citep{BassilyST14}. Further, such dependence on the dimension is optimal in both settings \citep{BassilyST14}.

The dependence on the dimension is not necessary for non-private learning in this setting. In addition, we can exploit known stability analyses to reduce (or even eliminate) the need to use the aggregation step. By plugging the known stability results based on strong convexity and/or \cite{BousquettE02,ShwartzSSS10,HardtRS16}, we demonstrate that convex regression problems of this type can be solved with $\eps$-differentially private prediction using $O(1/(\eps\alpha^2))$ examples (Cor.~\ref{cor:private-convex}). We also demonstrate that smoothness of the loss function $\ell$ can be used to improve the dependence on $\eps$ (Cor.~\ref{cor:private-convex-smooth}).
We note that stability of the optimal solution of a strongly convex problem has been used to achieve differential privacy in multiple prior works starting with the pioneering work of \citet{ChaudhuriMS11}. Stability of gradient descent on convex smooth functions has also been recently used to obtain privacy guarantees \citep{WuLKCJN17}.

\subsection{Beyond aggregation: learning thresholds}
The class of linear thresholds $\Thr$ is defined over a subset of reals and consists of  indicator functions of ``$x \geq a$" for all $a \in \R$. Without loss of generality, we consider such functions over the set $[N]=\{1,\ldots,N\}$.  While the class is very simple, learning it with privacy has proved to be rather challenging and some basic questions are still not fully resolved  \citep{BeimelKN:10,ChaudhuriHsu:11,BeimelNS:13,FeldmanXiao15,BunNSV15}. It is known that $\eps$-differentially private PAC learning of $\Thr$  requires $\Omega(\log(N)/(\eps\alpha))$ examples \citep{FeldmanXiao15} and {\em proper} $(\eps,\delta)$ differentially private PAC learning requires $\Omega(\log^*(N)/(\eps\alpha))$ examples \citep{BunNSV15} (no lower bounds for non-proper learning and $\delta>0$ case are known). Note that the VC dimension of this class is just $1$.

We give an $\eps$-differentially private prediction algorithm for agnostic learning of this class with the following guarantee:
\begin{thm}
\label{thm:thr-intro}
For any $\alpha,\eps > (0,1]$ and $N\in \N$, there exists an $\eps$-differentially private prediction algorithm $M$ that given $n \geq \frac{12 \ln(2/\alpha)}{\alpha \eps}$ examples from an arbitrary distribution $\cP$ over $[N]\times \zo$ guarantees:
$$\E_{S \sim \cP^n}\lb \err_\cP(M(S)) \rb  \leq e^\eps \cdot (\opt_\cP(\Thr) + \alpha) .$$
\end{thm}
Note that this statement implies an upper bound of $n=O(\ln(1/\alpha)/(\alpha\eps))$ in the realizable case when $\opt_\cP(\Thr) = 0$ and also an upper bound of $n=O(\ln(1/\alpha)/(\alpha\eps) + \ln(1/\alpha)/\alpha^2)$ in the agnostic setting. The $\tilde O(1/\alpha^2)$ term arises from having to set $\eps < \alpha$ to ensure that the expected error is at most $\opt_\cP(\Thr) + O(\alpha)$.
Our algorithm can also handle unions of $k$ intervals (at the expense of an additional factor $k$ in the sample complexity).

At a high level our algorithm works as follows. First, the examples are sorted. To determine the probability with which to output $1$ on point $x$ the algorithm traverses the examples on points smaller than $x$ in increasing order. Starting from bias $1/2$ the algorithm increases or decreases the current bias by a factor of (roughly) $e^\eps$ for each example it traverses. The bias is increased if the label of the example is 1 and decreased otherwise. Importantly, the bias is projected back to the interval $[\alpha,1-\alpha]$ after each update. The algorithm outputs 1 with probability obtained at the end of this process. While the prediction privacy of our algorithm is relatively easy to establish, the analysis of its error is more delicate and we are not aware of similar algorithms having been proposed for this problem. Furthermore, our analysis only bounds the empirical error of this algorithm. The hypothesis produced by the algorithm is sufficiently complicated that it would not be possible to ensure generalization using VC dimension or similar techniques. Remarkably, the fact that our algorithm is prediction private allows us to prove that it generalizes.

\subsection{Generalization}
It has been known for a while that differential privacy is a notion of stability and hence implies bounds on the expectation of generalization error. Recent work in the context of adaptive data analysis has substantially strengthened this connection, proving that differential privacy ensures generalization with high probability \citep{DworkFHPRR14:arxiv,BassilyNSSSU16,FeldmanS17}. Prediction privacy can also be seen as a notion of stability that is weaker than differential privacy but stronger than uniform prediction stability. We show how to derive relatively strong generalization guarantees from this notion of stability. These guarantees are stronger than those known for classical notions of stability (\eg \citep{BousquettE02,ShwartzSSS10}) but not as strong as those proved for differential privacy. Specifically, our generalization results (Lem.~\ref{lem:generalize-moment-bound}) imply that for every non-negative loss function $\ell$, a moment $k \geq 1$, and an $\eps$-differentially private prediction algorithm $M$:
$$\E_{S,S'\sim \cP^n,} \lb \lp \cE_{S'}[\ell(M(S))]\rp^k \rb \leq e^{k^2 \eps} \cdot \E_{S\sim \cP^n} \lb \lp\cE_{S}[\ell(M(S))]\rp^k \rb ,$$
where $\cE_{S}[\ell(M(S)]$ denotes the expected empirical loss of $M(S)$ on $S$. Note that on the left hand side we are bounding the average loss on an independently drawn set of examples $S'$ which is tightly concentrated around the expected loss $\cE_{\cP}[\ell(M(S)]$. For comparison, $\eps$-differential privacy gives a similar bound with $e^{k\eps}$ factor instead of $e^{k^2\eps}$ \citep{DworkFHPRR14:arxiv}. The bound above is stated using the $k=1$ version of this result. However this generalization bound implies that loss is also well concentrated. In Lemma~\ref{lem:generalize-high-prob} and Theorem~\ref{thm:thr-expectation} we give an example of how to derive high probability bounds on the generalization error from this moment bound.


\subsection{Related work}
\label{sec:related}
\citet{PathakRR10} consider secure and differentially private aggregation of non-private linear models held by multiple mistrusting parties. They achieve it by computing the average model and adding noise to it. They do not consider accuracy guarantees of their approach formally.

To the best of our knowledge, the privacy-preserving aggregation of non-private predictions to produce privacy-preserving predictions was first investigated by Bilenko, Dwork, Muthukrishnan, Rothblum, Thakurta and Wang in 2014\footnote{This was the core of a larger project on privacy-preserving click prediction that did not survive the closing of the Silicon Valley lab.}.   Bilenko \etal, obtained high levels of composition by exploiting the frequently high degree of (near) consensus among the predictions of the non-private models via a variant of the sparse-vector technique~\cite{DworkRoth:14}.  Our work shares the same goal of generating differentially private predictions. At the same time we formalize the general problem of learning with differentially private predictions and focus on the sample complexity of making a single prediction. In addition, we demonstrate approaches that go beyond privacy-preserving aggregation.


Aggregation of non-private models to produce labels while preserving privacy was also used in recent works of \citet{HammCB16} and, subsequently, \citet{PapernotAEGT17,papernot2018scalable} to give a new semi-supervised approach to differentially private learning. Specifically, their approach is predicated on availability of public {\em unlabeled} dataset $Z$. The dataset $Z$ is labeled using differentially-private aggregation of labels provided by models trained on the sensitive dataset $S$. The labeled data is used to train a new model. Since differential privacy is closed under post-processing, this new model is privacy-preserving for $S$ (but not for $Z$).
The works of Papernot et al. \citep{PapernotAEGT17,papernot2018scalable} deal primarily with techniques for accurately bounding the privacy parameters while ensuring accurate prediction on benchmark datasets. \citet{HammCB16} also formally examine additional error that noisy aggregation introduces and explicitly rely on stability of strongly-convex regression problems to provide formal guarantees for their approach. Their framework and the guarantees are incomparable to ours, and, in particular, they do not avoid dependence on the dimension.

In a recent and independent work, \citet{BassilyTT:18} consider the formal guarantees for answering a sequence of prediction queries using differentially private aggregation techniques. They demonstrate that given a non-private learning algorithm has error of at most $\alpha$ (such as in the PAC model), there exists an algorithm that answers $m$ prediction queries for points chosen i.i.d.~from the same distribution with error $O(\alpha)$ and privacy parameter $\eps$ scaling as $\sqrt{m \alpha} \cdot \log m$ (for comparison, a direct application of composition theorems for differential privacy implies $\sqrt{m}$ scaling for an arbitrary sequence of queries). They then analyze the sample complexity of semi-supervised (or, equivalently, label-private) learning algorithm that is obtained by labeling a public unlabeled dataset using their algorithm for answering prediction queries.

We remark that all these works do not examine the problem of private prediction itself and focus on the aggregation-based approaches. Recall that in private prediction, it is the privacy of the training data for the predictor (model) that is being protected. 


\paragraph{Organization:}
In Section \ref{sec:pac-learning} we provide additional details of our results for PAC learning. Results for agnostic learning and convex regression appear in Section \ref{sec:convex}. Section \ref{sec:thr} formally describes our algorithm for agnostic learning of thresholds and unions of intervals. We discuss the generalization properties of private prediction in Section \ref{sec:generalization}.

\section{Preliminaries}
\label{sec:prelims}
\paragraph{Differential privacy}
Differential privacy \citep{DworkMNS:06} relies on bounding the divergence between distributions output by the algorithm on neighboring datasets. Specifically, for two random variables $U$ and $V$ and $\delta > 0$ the $\beta$-approximate max-divergence is defined as (\eg \citep{DworkRoth:14}):
$$D_\infty^\delta(U \| V) \doteq \sup_{O \subseteq \mbox{supp}(U);\ \pr[U\in O] >\delta} \ln \frac{\pr[U\in O]-\delta}{\pr[V\in O]}  .$$ A randomized algorithm $M:X^n\to Y$ is said to be $(\eps,\delta)$-differentially private if for all pairs $S,S'\in X^n$ that differ on a single element,  $D_\infty^\delta(M(S) \| M(S')) \leq \eps$. We note that our definitions and many of the results can be immediately extended to more refined notions of differential privacy such as those based on Renyi divergence \citep{BunS16,Mironov17}.

The group privacy property of differential privacy (\eg \citep{DworkRoth:14}) implies the that prediction privacy has the analogous property.
\begin{lem}[Group privacy]
\label{lem:group-privacy}
Let $M:(X\times Y)^n\times X \to Y$ be an $(\eps,\delta)$-differentially private prediction algorithm and $k \in \mathbb{N}$.
For all pairs of data sets $S,S'\in (X\times Y)^n$ differing in at most $k$ elements and all $x\in X$:
$$D_\infty^{ke^{\epsilon(k-1)}\delta}(M(S,x)\|M(S',x)) \leq k \eps\ .$$
\end{lem}

\subsection{Learning models}
\begin{defn}
\label{def:pac-model}
An algorithm $A$ PAC learns a concept class $C$ from $n$ examples if for every $\alpha > 0, \beta>0$, $f\in C$ and distribution $\D$ over $X$, $A$ given access to $S = \{(x_i, \ell_i)\}_{i \in [n]}$ where each $x_i$ is drawn randomly and independently from $\D$ and $\ell_i = f(x_i)$, outputs, with probability at least $1-\beta$ over the choice of $S$ and the randomness of $A$, a function $h:X\to \zo$ such that $\pr_{x\sim \D}[f(x) \neq h(x)] \leq \alpha$.
\end{defn}

For a Boolean function $h$ and a distribution $\cP$ over  $X \times \zo$ let $\err_\cP(h) = \pr_{(x,\ell) \sim \cP}[h(x) \neq \ell]$. Define $\opt_\cP(C) = \inf_{h \in  C}\{\err_\cP(h)\}$.
\citet{KearnsSS:94} define agnostic learning as follows.
\begin{defn}
\label{def:agnostic-model}
An algorithm $A$ {\em agnostically} learns a concept class $C$ from $n$ examples if for every $\alpha > 0, \beta>0$, distribution $\cP$ over $X \times \zo$, $A$, given access to $S = \{(x_i, \ell_i)\}_{i \in [n]}$ where each $(x_i,\ell_i)$ is drawn randomly and independently from $\cP$, outputs, with probability at least $1-\beta$ over the choice of $S$ and the randomness of $A$, a function $h:X\to \zo$ such that $\err_\cP(h) \leq \opt_\cP(C) + \alpha$.
\end{defn}

\iffull
\section{Prediction privacy via subsampling and uniform stability}
In this section we describe two variants of the baseline approach to obtaining prediction privacy. The baseline approach is based on a well-known observation that stability to replacement (or deletion) of a point can be improved by partitioning the dataset $S$ into several subsamples $S_1,\ldots,S_\ell$ running a learning algorithm on each of those subsamples to obtain predictors $f_1,\ldots,f_\ell$ and then aggregating these predictors in a stable way.
The first variant we describe is specialized to the simpler realizable case of classification. The second one is a generic model averaging that works for arbitrary convex loss functions. This case can also be used to obtain guarantees for agnostic learning of Boolean functions. In this case we will also explicitly use uniform prediction stability properties of the algorithm to derive its privacy guarantees.

\subsection{PAC Learning}
\fi
\label{sec:pac-learning}
Our algorithm for PAC learning applies a soft majority rule to the outputs of $f_1,\ldots,f_r$. Specifically, on a point $x$ it will output a label $b$ with probability proportional to $e^{\eps |\{i \in [r]: f_i(x) =b\}|}$. This approach works well for PAC learning since all predictors agree very well with the true labeling function. In particular, if each of the predictors has error of at most $\alpha$, then the fraction of points on which more than $1/4$ of the predictors output the wrong label cannot be more than $4\alpha$. Therefore the prediction of the soft majority will be close to the true label on all but the $4\alpha$ fraction of the points.
\begin{thm}
\label{thm:pac-reduction}
Let $C$ be a class of Boolean functions over $X$. Let $A$ be a PAC learning algorithm for $C$ that uses $m(\alpha,\beta)$ samples to learn with error $\alpha$ and confidence parameter $\beta$. For every $\eps>0$, there exists an $\eps$-differentially private prediction algorithm $M$ that PAC learns $C$ using $n = r \cdot m(\alpha/4,\beta/r)$ examples, where $r=\lceil 6 \ln(4/\alpha)/\eps \rceil$.
\end{thm}
\iffull
\begin{proof}
We denote by $c\in C$ the unknown labeling function and by $\D$ the unknown distribution over $X$.
We let $r= \lceil 6 \ln(4/\alpha)/\eps \rceil$ and $n' = m(\alpha/4,\beta/t)$. Given a set $S$ of $n= r \cdot n'$ examples we split them randomly into $r$ disjoint subsets of size $n'$. We now run $A$ with error parameter set to $\alpha/4$ and confidence parameter to $\beta/r$ on each of those sets to obtain $r$ functions $f_1, \ldots,f_r$.
On an input point $x$ let $v(S,x) = 2|\{i \in [r]: f_i(x) = 1\}|-r$. Our algorithm outputs $1$ with probability $\frac{e^{\eps v(S,x)/2}}{1+e^{\eps v(x)/2}}$ and $0$ otherwise.

We first observe that $M$ is an $\eps$-differentially private prediction algorithm. This follows easily from observing that changing a single example can change only a single function $f_i$. Further, such change can change the value $v(S,x)$ by at most $2$. Namely, for any pair of neighboring dataset $S, S'$, $|v(S,x) - v(S',x)| \leq 2$. Now the privacy guarantees follow immediately from the definition of the output distribution of our algorithm being as: output $1$ with probability $\frac{e^{\eps \cdot v(S,x)/2}}{1+e^{\eps \cdot v(S,x)/2}}$ and $0$ with probability $\frac{1}{1+e^{\eps \cdot v(S,x)/2}}$ and the fact that for arbitrary real $a$, $b$,
$$ \frac{\frac{e^a}{1+e^a}}{\frac{e^b}{1+e^b}} = e^{a-b} \cdot \frac{1+e^b}{1+e^a}  \leq e^{|a-b|} \mbox{    and     }  \frac{1+e^b}{1+e^a} \leq e^{|a-b|} . $$

We now analyze the accuracy of our algorithm. Using the union bound, we know that with probability at least $1-\beta$, for every $i\in [r]$,
$\pr_\D[f_i(x) \neq c(x)] = \E_\D \lb |f_i(x)- c(x)| \rb \leq \alpha/4$. This means that
$$\E_\D \lb \sum_{i\in[r]} |f_i(x)- c(x)| \rb \leq \alpha r /4 .$$
By Markov's inequality this implies that
\equ{\pr_\D \lb \sum_{i\in[r]} |f_i(x)- c(x)|  \geq r/3\rb \leq 3\alpha/4 .\label{eq:few-bad-points}}

Now we claim that for every $x$ such that $\sum_{i\in[r]} |f_i(x)- c(x)|  \leq r/3$ we have that $\pr_M[M(S,x)\neq c(x)] \leq \alpha/4$. If $c(x)=1$ then $$v(S,x) =  2 (r-\sum_{i\in[r]} |f_i(x)- c(x)|)-r \geq \frac{r}{3}. $$ This implies that $$\pr_M[M(S,x)\neq 1] \leq \frac{1}{1+e^{\eps r/6}} \leq e^{-\eps r/6} \leq e^{-\ln(\alpha/4)} = \alpha/4.$$ Similarly, if $c(x)=0$ we get that $v(S,x) \leq -r/3$ and $\pr_M[M(S,x)\neq 0] \leq \alpha/4$.

Combining the last claim with inequality \eqref{eq:few-bad-points} we obtain that $\pr_{\D,M}[c(x) \neq M(S,x)] \leq \alpha$.
\end{proof}
\fi

Standard bounds on the sample complexity of PAC learning (e.g.~\citep{KearnsVazirani:94}) state that $n= O\lp \frac{d\log(1/\alpha) + \log(1/\beta)}{\alpha}\rp$ examples suffice to PAC learn a class $C$ of VC dimension $d$. Plugging this into our reduction we obtain that for every concept class $C$ there exists a differentially private prediction algorithm that PAC learns the class $C$ given $n=\tilde O(d/(\eps \alpha))$ examples.
\iffull
\begin{cor}[Cor.~\ref{cor:pac} restated]
Let $C$ be a class of Boolean functions of VC dimension $d$. Then for all $\alpha,\beta,\eps > 0$ there exists an $\eps$-differentially private prediction algorithm $M$ that PAC learns $C$ with error $\alpha$ and confidence $1-\beta$ given $n = O\lp\frac{d \log^2(1/\alpha) + \log(1/\alpha)\log(\log(1/\alpha)/(\eps\beta))}{\eps\alpha}\rp$ examples.
\end{cor}
\fi
We demonstrate that this upper bound is essentially tight.
\iffull
\begin{thm}[Thm.~\ref{thm:pac-lower-bound} restated]
Let $C$ be a class of Boolean functions of VC dimension $d$. Then for all $\alpha,\eps > 0$, any $(\eps,\eps/3)$-differentially private prediction algorithm $M$ that PAC learns $C$ with error $\alpha$ and confidence $1/12$ requires $n\geq d/(32\eps\alpha)$ examples.
\end{thm}
\fi
\begin{proof}
We first deal with the case $\alpha = 1/4$. The reduction to general $\alpha$ is standard and is briefly described below.

Let $a_1,\ldots,a_d \in X$ be the set of points shattered by $C$. For convenience we refer to these points as $\{1,2,\ldots,d\}$. For a vector $b=(b_1,\ldots,b_d)\in \zo^d$ we denote by $f_b$ the function in $C$ that satisfies: for all $i\in [d], f_b(i) = b_i$. Let $\D$ be the uniform distribution over $[d]$.

Let $M$ be the $(\eps,\eps/3)$-differentially private prediction algorithm for learning $C$.
\eat{
To simplify our argument we first symmetrize $M$ in two ways. First we pick a uniform and random label mask $z \in \zo^d$ and then we pick a random and uniform permutation on $d$ elements. Now $M'$ works as follows: it takes the dataset $S=(x_1,y_1),\ldots,(x_n,y_n)$ and an input point $x$ and replaces them with the dataset $S_{z,\sigma} = ((\sigma(x_1), y_1 \oplus z_{x_1}),\ldots,(\sigma(x_n),y_n \oplus z_{x_n}))$ and input point $\sigma(x)$, where $\oplus$ denotes the XOR operation. It then runs $M$ on $S_{z,\sigma}$ and $\sigma(x)$. Let $v$ denote the output of $B$. $M'$ outputs $v \oplus z_x$.

Note that each example in dataset $S_{z,\sigma}$ is sampled i.i.d. from distribution $(\D,f_{b'})$, where $b'=\sigma^{-1}(b\oplus z)$ (that is for $i\in [d]$, $b'_{\sigma(i)} = b_i \oplus z_i$). Therefore $M$ should learn it with accuracy $\geq 1/4$ and success probability $\geq 1/12$. This means that $M'$ will have the same accuracy guarantees on when learning $f_b$. It is also easy to see that this symmetrization preserves the prediction privacy parameters of the algorithm. Our symmetrization implies that for every, $b\in \zon$ and $i,j \in [d]$,
\equ{\pr_{S \sim (\D,f_b)^n,M'}[M'(S,i) \neq b_i] = \pr_{S \sim (\D,f_b)^n,M'}[M'(S,j) \neq b_j] \label{eq:sym-index}.}
Similarly,
\equ{\pr_{S \sim (\D,f_b)^n,M'}[M'(S,i) \neq b_i] = \pr_{S \sim (\D,f_{b^{\oplus i}})^n,M'}[M'(S,i) = b_i] \label{eq:sym-label},}
where $b^{\oplus i}$ denotes $b$ with the bit in $i$-th index flipped.
$\E_{i\sim \D} [p_i] \leq 1/3$ and, in particular, all $p_i$'s are at most $1/3$.

Our symmetrization (eq.~\eqref{eq:sym-index}) implies that $p_i = p_j$ for all $i,j\in[d]$. At the same time

For $v \in \zo$, we now define $p_{i,v}$ as the same probability conditioned on $b_i = v$, namely
$$p_{i,v} \doteq \pr_{b\sim\zo^d_{|b_i=v},\ S \sim (\D,f_b)^n,M'}[M'(S,i) \neq v],$$
where $b\sim\zo^d_{|b_i=v}$ denotes a random and uniform choice of $b$ from $\zo^d$ restricted to $b_i =v$.

The label symmetrization of $M'$ (eq.~\eqref{eq:sym-label}) implies that $p_{i,0} = p_{i,1}$.

}
Now consider the expected prediction accuracy of $M$ on a point $i\in [d]$, where the expectation is taken over the following process: $b\in \zo^d$ is chosen randomly, a dataset of size $n$ is generated from $\D$ labeled by $f_b$ and then $M$ is run on $S$ and $i$. Namely,
$$p_i \doteq \pr_{b\sim\zo^d,S \sim (\D,f_b)^n,M}[M(S,i) \neq b_i].$$
The accuracy and confidence guarantees of $M$ imply that
$$\E_{i\sim \D} [p_i] = \E_{b\sim\zo^d,\ S \sim (\D,f_b)^n} \lb \E_{i\sim \D} \lb \pr_{M}[M(S,i) \neq b_i] \rb \rb \leq \alpha + \beta = \fr{4} + \fr{12} = \fr{3}. $$
This means that there exists $i$ such that $p_i\leq 1/3$ and we fix $i$ to this value for the rest of the argument.

Let $S^{\oplus i}$ denote the dataset in which all the examples for point $i$ have their label flipped. By group prediction privacy of $M$ (Lemma \ref{lem:group-privacy}) we know that for every $v\in \zo$, $$\pr[M(S,i) = v] \leq e^{\eps t} \pr_{M}[M(S^{\oplus i},i) = v] + te^{(t-1)\eps}\delta ,$$ where $t$ is the number of points $i$ in the dataset.

Now if we assume, for the sake of contradiction, that $n \leq d/(8\eps)$ then (for $d$ larger than some fixed constant) with probability at least $1/24$ over the choice of $S$, $S$ includes at most $s \doteq 1/(4 \eps)$ points equal to $i$. Using that $e^{\eps s} = e^{1/4} \leq 3/2$ and $\delta = \eps/3$, this implies that
  \alequ{\pr_{S\sim (\D,f_b)^n,M}[M(S,i) = v] & \leq e^{\eps s} \pr_{S\sim (\D,f_b)^n,M}[M(S^{\oplus i},i) = v] + s \cdot e^{(s-1)\eps}\delta  + \fr{24} \nonumber\\
  & < \frac{3}{2} \cdot \pr_{S\sim (\D,f_b)^n,M}[M(S^{\oplus i},i) = v] + \fr{4\eps} \cdot \frac{3}{2} \cdot \frac{\eps}{3} + \fr{24}  \nonumber\\
  & =  \frac{3}{2} \cdot \pr_{S\sim (\D,f_b)^n,M}[M(S^{\oplus i},i) = v] + \fr{6}. \label{eq:group-privacy}}
Observe that for every $b$ and $S \sim (\D,f_b)^n$, $S^{\oplus i}$ is distributed identically to $S \sim (\D,f_{b^{\oplus i}})^n$. This implies that,
\alequ{\pr_{b\sim\zo^d,\ S\sim (\D,f_b)^n,M}[M(S^{\oplus i},i) = b_i] &= \pr_{b\sim\zo^d, S\sim (\D,f_{b^{\oplus i}})^n,M}[M(S,i) = b_i] \nonumber \\
& = \pr_{b\sim\zo^d, S\sim (\D,f_b)^n,M}[M(S,i) \neq b_i] \nonumber \\
& = p_i. \label{eq:label-sym}}

By plugging equations \eqref{eq:group-privacy} and \eqref{eq:label-sym} into the definition of $p_i$ we obtain that:
\alequn{1-p_i &= \pr_{b\sim\zo^d, S \sim (\D,f_b)^n,M}[M(S,i) = b_i] \\
& < \frac{3}{2} \cdot \pr_{b\sim\zo^d, S\sim (\D,f_b)^n,M}[M(S^{\oplus i},i) = b_i] + 1/6 \\
& = \frac{3}{2} \cdot p_i + 1/6.}
This cannot hold when $p_i \leq 1/3$, implying that $n > d/(8\eps)$.

Finally we reduce the general $\alpha$ case to the analysis for $\alpha =1/4$ in the standard way (\eg \citep{KearnsVazirani:94,Shalev-ShwartzBen-David:2014}). We let $\D_\alpha$ be the distribution that outputs the point with index $d$, with probability $1-4\alpha$ and outputs a uniformly and randomly chosen $i \in [d-1]$ with probability $4\alpha$. Achieving error of $\alpha$ on $\D_\alpha$ requires achieving error of $1/4$ on the uniform distribution on $[d-1]$. Only approximately $4\alpha$ fraction of the examples will be useful for obtaining low error relative to the uniform distribution on $[d-1]$ and therefore the reduction multiplies the lower bound by $\Omega(1/\alpha)$. More formally, we consider only target functions $f_b$ where $b_d=0$. Therefore for all target functions, examples on point $d$ will be identical. Now given that $n \leq d/(32\alpha\eps)$, (and for $d$ larger than some fixed constant) with probability at least $1/24$ over the choice of $S\sim (\D_\alpha,f_b)^n$, $S$ includes at most $1/(4 \eps)$ points equal to $i$ as before. Hence the rest of the argument is essentially identical.
\end{proof}

\iffull
\subsection{Learning with convex losses and stability}
\label{sec:convex}
We now deal with the general setting of minimizing convex losses. Specifically, these are problems in which the goal is to minimize the expected loss function $\E_{(x,y)\sim \cP}[\ell(f(x),y)]$, where $\ell$ is a convex function in the first parameter. Note that learning of Boolean functions is a special case in which we use $\ell(a,b) = |a-b|$.

To deal with this case we will rely on (non-private) learning algorithms that are prediction stable in the usual numerical sense. That is
\begin{defn}
A learning algorithm $A$ is uniform replace-one (RO) prediction stable with rate $\gamma$ if for all datasets $S,S'\in (X\times Y)^n$ that differ in a single element and any $x \in X$,
$$ |A(S,x) - A(S',x)| \leq \gamma ,$$
where $A(S,\cdot)$ denotes the function output by $A$ on dataset $S$.
\end{defn}

This notion is closely-related to the standard uniform replace-one stability \citep{BousquettE02,ShwartzSSS10} which bounds the change in loss $|\ell(A(S,x),y) - \ell(A(S',x),y)| \leq \gamma$ for all $x,y$. Crucially, the analyses of uniform loss stability that we are aware of implicitly prove bounds on prediction stability. Hence such analyses can be adapted to our applications (we provide some examples below).

Low sensitivity of the value predicted at each point implies that addition of Laplace or Gaussian noise suffices to obtain a differentially private prediction algorithm. The additional error due to noise can be controlled for Lipschitz loss functions. Somewhat stronger bounds on the additional error can be shown if the loss function is smooth (that is, its derivative is Lipschitz-bounded).
\begin{lem}
\label{lem:stability-to-privacy}
Let $\ell:\R\times Y \rar \R$ be a loss function convex in the first parameter. Let $A$ be a uniform RO prediction stable algorithm with rate $\gamma$. For every $\eps > 0$, there exists an $\eps$-differentially private prediction algorithm $M$ such that for every dataset $S\in (X\times Y)^n$ and any probability distribution $\cP$ over $X\times Y$:
\begin{enumerate}
 \item if $\ell(\cdot,y)$ is $L_\ell$-Lipschitz in the first parameter for all $y\in Y$ then
$$\cE[\ell(M(S))] \leq \cE[\ell(A(S))] + L_\ell \cdot \gamma/\eps .$$
 \item if $\ell(\cdot,y)$ is $\sigma$-smooth in the first parameter for all $y\in Y$ then
$$\cE[\ell(M(S))] \leq \cE[\ell(A(S))] + \sigma^2\gamma^2/\eps^2 .$$
\end{enumerate}
\end{lem}
\begin{proof}
Given $S$ and $x$ let $v$ be the output of $A$ on $S$ applied to $x$. We output $v + \zeta$, where $\zeta$ is distributed according to Laplace distribution with scale $\eps/\gamma$. By definition of uniform RO prediction stability and standard properties of the Laplace noise addition (e.g. \citep{DworkRoth:14}), this algorithm is $\eps$-differentially private. To obtain the claimed upper bound on the expected loss observe that if  $\ell$ is $L_\ell$-Lipschitz then
for any $S,x$ and $y$, 
$$\E_M[\ell(A(S,x) +\zeta,y)] \leq \ell(A(S,x),y) + \E_M[L_\ell \cdot |\zeta|] = \ell(A(S,x),y) +  L_\ell \cdot \gamma/\eps, $$ where we have used the fact that $|\zeta|$ is distributed according to an exponential distribution with rate $\gamma/\eps$.
This upper bound holds pointwise and therefore for any distribution $\cP$ over $X\times Y$.

If $\ell$ is $\sigma$-smooth, then by the definition of smoothness:
 for any $x$ and $y$, $$\ell(A(S,x)+\zeta,y) \leq \ell(A(S,x),y) + \ell'(A(S,x),y)\cdot \zeta + \frac{\sigma}{2} \cdot \zeta^2 .$$ Using the fact that $\E[\zeta] = 0$ and $\E[\zeta^2] = 2 \gamma^2/\eps^2$, we obtain
 $$\E_M[\ell(A(S,x)+\zeta,y)] \leq \ell(A(S,x),y) + \gamma^2\sigma/\eps^2. $$
\end{proof}

Naturally, many learning algorithms are not sufficiently prediction stable to ensure that the additional error due to noise is sufficiently small. However it is easy to see that it is possible to amplify stability by averaging the predictions obtained on disjoint subsamples. Convexity of the loss function then implies that such averaging will preserve the bounds on the expected loss. Specifically, the following lemma follows immediately from the argument above.
\begin{lem}
\label{lem:amplify-stability}
Let $A$ be a learning algorithm that outputs a real-valued function on $X$, is uniform RO prediction stable with rate $\gamma$ and uses $n$ samples. For any $r \in \N$ there exists a learning algorithm $A'$ that is uniform RO prediction stable with rate $\gamma' = \gamma/r$ that uses $n \cdot r$ samples. Further, for any loss function $\ell(\cdot,\cdot)$ convex in the first parameter, if for a distribution $\cP$ over $X\times Y$, $A$ has the guarantee that
$\E_{S\sim \cP^n}\lb\cE_\cP[\ell(A(S))] \rb \leq v$ for some value $v$ that may depend on $\cP$ and the parameters of the learning problem then $\E_{S'\sim \cP^{rn}}\lb\cE_{\cP}[\ell(A'(S'))] \rb \leq v .$
Alternatively, if for some $\beta > 0$,
$$\pr_{S\sim \cP^n}\lb\cE_\cP[\ell(A(S))] \geq v \rb \leq \beta $$ then
$$\pr_{S'\sim \cP^{rn}}\lb\cE_\cP[\ell(A'(S'))]  \geq v \rb \leq r\beta .$$
The running time of $A'$ is $r$ times the running time of $A$.
\end{lem}


\subsection{Agnostic learning}
We now spell out immediate corollary of Lemmas \ref{lem:amplify-stability} and \ref{lem:stability-to-privacy} to agnostic learning of Boolean functions. Agnostic learning of Boolean functions reduces to learning of a real-valued function $f$ with absolute loss $\ell(a,y) = |y-a|$. Note that a real-valued prediction $f(x)$ can also be equivalently thought of as predicting $1$ with probability $p$, where $p$ is equal to $f(x)$ projected to the interval $[0,1]$. The expected disagreement of such prediction on $y\in \zo$ is upper bounded by $|y-f(x)|$. This loss function is convex and 1-Lipschitz in the first parameter. Further, any learning algorithm that outputs a Boolean function is uniform RO prediction stable at the trivial rate $1$. Hence to ensure that the additional error in Lemma \ref{lem:stability-to-privacy} is at most $\alpha$ we need to amplify the stability to $\alpha \cdot \eps$. This requires $r = 1/(\alpha\eps)$ subsamples. Therefore overall we obtain the following general reduction for agnostic learning of Boolean function.
\begin{cor}
\label{cor:agnostic-reduction}
Let $C$ be a class of Boolean functions over $X$. Let $A$ be an agnostic learning algorithm for $C$ that uses $m(\alpha,\beta)$ samples to learn with excess error $\alpha$ and confidence parameter $\beta$. For every $\eps\in(0,1]$, there exists an $\eps$-differentially private prediction algorithm $M$ that agnostically learns $C$ given $n = 2\cdot m(\alpha/2,2\beta\alpha\eps)/(\alpha\eps)$ examples.
\end{cor}
As in the case of PAC learning, this reduction allows us to upper bound the sample complexity of private prediction for agnostic learning of $C$.
Specifically, $n= O\lp\frac{d + \log(1/\beta)}{\alpha^2}\rp$ samples suffice to agnostically learn a class VC dimension $d$ (e.g. \citep{Shalev-ShwartzBen-David:2014}). Plugging this into our reduction we get:
\begin{cor}[Cor.~\ref{cor:agnostic-sample-complexity-intro} restated]
Let $C$ be a class of Boolean functions of VC dimension $d$. Then for all $\alpha,\beta,\eps \in (0,1]$ there exists an $\eps$-differentially private prediction algorithm $M$ that agnostically learns $C$ with excess error $\alpha$ and confidence $1-\beta$ given $n = O\lp\frac{d  + \log(1/(\alpha\beta\eps))}{\eps\alpha^3}\rp$ examples.
\end{cor}

\subsection{Applications to convex regression problems}
We now apply this general approach to convex regression problems. Specifically, problems of the form:
$$\min_{w \in \K} \E_{(x,y)\sim \cP}[ \ell(f(w,x),y)],$$ where $\K \subset \R^d$ is some convex body and $\ell(f(\cdot,x),y)$ is a convex function over $\K$ for all $(x,y) \in X\times Y$. For simplicity, we will restrict our attention to the case when $\K$ is a subset of the Euclidean ball of radius $R$ which we denote by $\B_2^d(R)$. Several classes of such problems are known to be solvable efficiently by uniform RO prediction stable algorithms.
Our result will be based on the following upper bound on prediction stability of strongly convex optimization that is implicit in \citep{BousquettE02,ShwartzSSS10}.
\begin{thm}[\citep{ShwartzSSS10}]
\label{thm:ssss}
Let $\K \subseteq \B_2^d(R)$ be a convex body, $\{f(\cdot, x) \cond x\in X\}$ be a family of $L_f$-Lipschitz functions over $\K$, $\ell:\R\times Y\rar \R$ be convex, $L_\ell$-Lipschitz loss function and $\ell(f(\cdot, x),y)$ be $\lambda$-strongly convex for all $(x,y) \in X\times Y$. For a dataset $S \in (X\times Y)^n$ let $w_S$ denote the empirical minimizer of loss on $S$: $w_S =\argmin_{w \in \K} \sum_{(x,y)\in S}[\ell(f(w,x),y)]$. Then the algorithm that given $S$, outputs a function $f(w_S,\cdot)$ is uniform RO prediction stable with rate $\frac{4 L_f^2 \cdot L_\ell}{\lambda n}$. Further, for every distribution $\cP$ over $X\times Y$:
$$\E_{S\sim \cP^n} \lb \E_\cP[\ell(f(w_S,x),y)] \rb  \leq  \min_{w \in \K} \E_\cP[\ell(f(w,x),y)] + \frac{4 L_f^2 \cdot L_\ell^2}{\lambda n}.$$
\end{thm}
We remark that this version may appear somewhat different from the results in \citep{ShwartzSSS10} as they consider a single convex loss function with Lipshitz constant $L$ that gives the loss of the model with parameter $w$ on an example. Our statement follows from noting that their work proves that
for any pair of datasets $S$ and $S'$ that differ in a single element, $\|w_S - w_{S'}\|_2 \leq \frac{4L}{\lambda n}$. This implies that for all $x$, $$|f(w_S,x) - f(w_{S'},x)| \leq \frac{4L \cdot  L_f}{\lambda n} \leq \frac{4 L_f^2 \cdot L_\ell}{\lambda n} .$$

By combining this result with Lemma \ref{lem:stability-to-privacy} we get the following private prediction algorithm.
\begin{cor}
\label{cor:private-strongly-convex}
Let $\K \subseteq \B_2^d(R)$ be a convex body, $\{f(\cdot, x) \cond x\in X\}$ be a family of $L_f$-Lipschitz functions over $\K$, $\ell:\R\times Y\rar \R$ be convex, $L_\ell$-Lipschitz loss function and $\ell(f(\cdot, x),y)$ be $\lambda$-strongly convex for all $(x,y) \in X\times Y$.
For every $\eps > 0$, there exists an $\eps$-differentially private prediction algorithm $M$ that for any probability distribution $\cP$ over $X\times Y$ satisfies:
$$\E_{S\sim \cP^n}\left[\cE_{\cP}[\ell(M(S))]\right] \leq \min_{w \in \K} \E_\cP[\ell(f(w,x),y)] + \frac{4 L_f^2 \cdot L_\ell^2}{\lambda n} \cdot \left(1 + \fr{\eps}\right).$$
\end{cor}

Corollary \ref{cor:private-strongly-convex} requires strong convexity to obtain meaningful guarantees. However, as pointed out in \citep{ShwartzSSS10}, it is possible to add a strongly convex regularizing term $\lambda \|w\|^2$ to the objective function that has sufficiently small effect on the loss function while ensuring stability (and generalization). Specifically, by setting $\lambda = \frac{2 L_fL_\ell}{R \sqrt{n\eps/(1+\eps)}}$ the objective function will change by at most $\frac{2 R L_fL_\ell}{\sqrt{n\eps/(1+\eps)}}$ since $w$ is assumed to be in a ball of radius $R$. Plugging this value of $\lambda$ into Corollary \ref{cor:private-strongly-convex} and accounting for the additional error we get:
\begin{cor}
\label{cor:private-convex}
Let $\K \subseteq \B_2^d(R)$ be a convex body, $\{f(\cdot, x) \cond x\in X\}$ be a family of $L_f$-Lipschitz functions over $\K$, $\ell:\R\times Y\rar \R$ be convex, $L_\ell$-Lipschitz loss function and $\ell(f(\cdot, x),y)$ be convex for all $(x,y) \in X\times Y$.
For every $\eps > 0$, there exists an $\eps$-differentially private prediction algorithm $M$ that for any probability distribution $\cP$ over $X\times Y$ satisfies:
$$\E_{S\sim \cP^n}\left[\cE_{\cP}[\ell(M(S))]\right] \leq \min_{w \in \K} \E_\cP[\ell(f(w,x),y)] + \frac{4 \cdot R\cdot L_f \cdot L_\ell}{\sqrt{n \eps/(1+\eps)}} .$$
\end{cor}

Somewhat stronger results can be obtained for regression problems in which the loss function is also smooth. In this case we can rely on the stability of gradient descent for smooth functions implicit in \citep{HardtRS16}. This result applies to the projected stochastic gradient descent algorithm. For concreteness, let PSGD$_\eta$ denote the algorithm that starting from $w_0$ being the origin, performs the following iterative updates for every $i\in [n]$: $$w_{i+1} \leftarrow \proj_\K(w_i + \eta \cdot \nabla \ell(f(w_i,x_i),y_i)) ,$$ where $\proj_\K$ denotes projection to $\K$. The algorithm returns the average iterate: $\bar{w}_S \doteq \fr{n} \sum_{i\in [n]} w_i$.
\begin{thm}[\citep{HardtRS16}]
\label{thm:hrs}
Let $\K \subseteq \B_2^d(R)$ be a convex body, $\{f(\cdot, x) \cond x\in X\}$ be a family of $L_f$-Lipschitz functions over $\K$, $\ell:\R\times Y\rar \R$ be convex, $L_\ell$-Lipschitz loss function and $\ell(f(\cdot, x),y)$ be convex and $\sigma$-smooth for all $(x,y) \in X\times Y$. For a dataset $S \in (X\times Y)^n$ let $\bar{w}_S$ denote the output of PSGD$_\eta$ for $\eta = R/(L_f L_\ell \sqrt{n})$. If $\sigma \leq 2/\eta$ then the algorithm that outputs $f(\bar{w}_S,\cdot)$ is uniform RO prediction stable with rate $R L_f/\sqrt{n}$. Further, for every distribution $\cP$ over $X\times Y$:
$$\E_{S\sim \cP^n} \lb \E_\cP[\ell(f(w_S,x),y)] \rb  \leq  \min_{w \in \K} \E_\cP[\ell(f(w,x),y)] + \frac{2 L_f \cdot L_\ell 
\cdot R}{\sqrt{n}}.$$
\end{thm}

Plugging this result into our framework we obtain the following stronger bound for convex and smooth functions. We will also additionally assume that the loss function $\ell(\cdot,y)$ is $\sigma_\ell$-smooth in the first parameter for all $y$. 
\begin{cor}
\label{cor:private-convex-smooth}
Let $\K \subseteq \B_2^d(R)$ be a convex body, $\{f(\cdot, x) \cond x\in X\}$ be a family of $L_f$-Lipschitz functions over $\K$, $\ell:\R\times Y\rar \R$ be convex, $L_\ell$-Lipschitz and $\sigma_\ell$-smooth loss function and $\ell(f(\cdot, x),y)$ be convex and $\sigma$-smooth for all $(x,y) \in X\times Y$. If $\sigma \leq 2L_f L_\ell \sqrt{n}/R$ then for every $\eps > 0$, there exists an $\eps$-differentially private prediction algorithm $M$ that for any probability distribution $\cP$ over $X\times Y$ satisfies:
$$\E_{S\sim \cP^n}\left[\cE_{\cP}[\ell(M(S))]\right] \leq \min_{w \in \K} \E_\cP[\ell(f(w,x),y)] + \frac{2 \cdot R\cdot L_f \cdot L_\ell}{\sqrt{n}} + \frac{\sigma_\ell \cdot R^2 \cdot L_f^2}{n \eps^2}.$$
\end{cor}
One way to interpret this result is that for sufficiently smooth loss functions, the error caused by noise become comparable to the statistical error when $\eps$ scales as $n^{-1/4}$. In other words, this level of differential privacy is obtained essentially for free.
We also remark that the assumption that $\ell(\cdot,y)$ is $\sigma_\ell$-smooth can also be used in Cor.~\ref{cor:private-convex} to obtain the same bound (up to a constant) as the one we got in Cor.~\ref{cor:private-convex-smooth}. Similarly, without this assumption Cor.~\ref{cor:private-convex-smooth} would give essentially the same bound as Cor.~\ref{cor:private-convex}.
\fi

\section{Learning of thresholds and unions of intervals}
\label{sec:thr}
We demonstrate a nearly optimal algorithm for agnostically learning the class of threshold functions on a line (and more generally unions of intervals).
For $N\in \mathbb{N}$ we consider threshold functions over $[N]$. Specifically, for $a\in [N+1]$, let $\theta_a$ denote the threshold function ``$x \geq a$'' over $[N]$ and let $\Thr_N$ denote the set of all $N+1$ threshold functions over $[N]$. More generally, we define a union of intervals by an increasing sequence of interval ends. Specifically,  for an increasing sequence of integer numbers $1\leq a_1<a_2<\cdots<a_k \leq N+1$, we define $\theta_{a_{[k]}}$ to be the function defined as follows: given $x\in M$ let $t(x)$ be the largest index such that $x\geq a_t$. Then $\theta_{a_{[k]}}(x)$ is equal to $1$ if and only if $t$ is odd. We denote by $\Thr_{N,k}$ the class of all functions of this type.

Our algorithm, referred to as the exponential projected walk, is described below. For convenience, we assume that the dataset $S=(x_1,y_1),\ldots,(x_n,y_n)$ is given in a sorted order, that is $x_i \leq x_j$ for all $i < j$. Further, we define $\proj_{[A,B]}(x)$ as the projection of $x$ onto interval $[A,B]$.
\begin{figure}[h!]
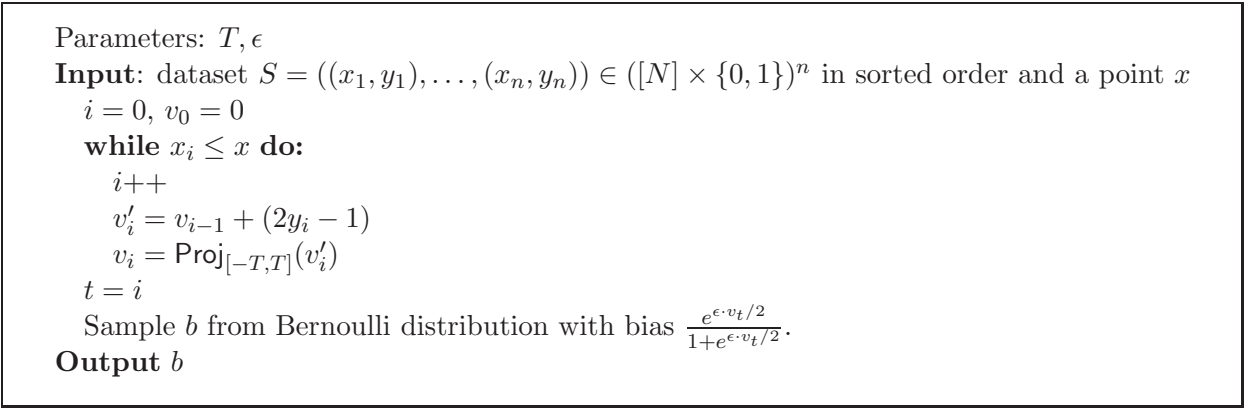

\begin{framed}
\begin{algorithmic}
\INDSTATE[0]{Parameters: $T,\eps$}
\INDSTATE[0]{{\bf Input}: dataset $S=((x_1,y_1),\ldots,(x_n,y_n)) \in ([N]\times \zo)^n$ in sorted order and a point $x$}
\INDSTATE[1]{$i=0$,\ $v_0 = 0$}
\INDSTATE[1]{{\bf while} $x_i \leq x$ \bf{do}:}
\INDSTATE[2]{$i$++}
\INDSTATE[2]{$v'_i = v_{i-1} + (2y_i-1)$}
\INDSTATE[2]{$v_i=\proj_{[-T,T]}(v'_i)$}
\INDSTATE[1]{$t = i$}
\INDSTATE[1]{Sample $b$ from Bernoulli distribution with bias $\frac{e^{\eps \cdot v_t/2}}{1+e^{\eps \cdot v_t/2}}$.}
\INDSTATE[0]{{\bf Output} $b$}
\end{algorithmic}
\end{framed}
\caption{$\epw(T,\eps)$: Exponential projected walk algorithm}\label{fig:ExpPW}
\end{figure}

We first prove that for any setting of parameters and $n$, the exponential projected walk is an $\eps$-differentially private prediction algorithm.
\begin{lem}
For any $n,T\in \mathbb{N}$, $\eps>0$, $\epw(T,\eps)$ is an $\eps$-differentially private prediction algorithm.
\end{lem}
\begin{proof}
Let $S=(x_1,y_1),\ldots,(x_n,y_n)$ be a dataset in a sorted order and let $S'$ be a dataset that differs from $S$ in a single element. There exist indices $i$ and $j$ such that $S'$ can be seen as removing the example $i$ and then inserting example $(x',y')$ into $j$-th position so that the resulting sequence of examples is still in the sorted order. Let $S^{-i}$ denote $S$ with $i$-th element removed. We will prove that for any $x$, $D_\infty(M(S,x) \| M(S^{-i},x)) \leq \eps/2$ and $D_\infty(M(S^{-i},x) \| M(S,x)) \leq \eps/2$. Note that, by removing element $j$ from $S'$ we also obtain $S^{-i}$. Hence our argument will imply that $$D_\infty(M(S,x) \| M(S',x)) \leq  D_\infty(M(S,x) \| M(S^{-i},x)) + D_\infty(M(S^{-i},x) \| M(S',x)) \leq \eps .$$

Let $M$ denote $\epw(T,\eps)$ and let $V(S,x)$ denote the value of $v_t$ at the end of running $M(S,x)$. Note that in order to prove the claim it is sufficient to prove that $|V(S,x)-V(S^{-i},x)|\leq 1$. As in the proof of Theorem \ref{thm:pac-reduction}, the claim then follows immediately from the definition of the output distribution of $\epw$ being as: output $1$ with probability $\frac{e^{\eps \cdot V(S,x)/2}}{1+e^{\eps \cdot V(S,x)/2}}$ and $0$ with probability $\frac{1}{1+e^{\eps \cdot V(S,x)/2}}$.

To show that $|V(S,x)-V(S^{-i},x)|\leq 1$ we observe that: for $x < x_i$ removal of $(x_i,y_i)$ does not affect the output of the algorithm. Hence $V(S,x)= V(S^{-i},x)$. For $x\geq x_i$ the values $v_0,\ldots,v_n$ of the walk on $S$ will have an additional step of length at most $1$ at index $i$. After that step the update points $(x_{i+1},y_{i+1}),\ldots, (x_t,y_t)$ will be identical for both sequences. Performing such an update step on two different values $u$ and $v$ does not increase the distance between the values. Hence at the end of the walk we obtain that $|V(S,x)-V(S^{-i},x)|\leq |V(S,x_i)-V(S^{-i},x_i)| \leq 1$.
\end{proof}

We now prove that our algorithm will achieve low empirical error.
Let $$\err_S(M(S)) \doteq \fr{n} \sum_{i\in[n]}\pr_{M}[M(S,x_i) \neq y_i]$$ and for a class of functions $C$ let
$\opt_S(C) \doteq \min_{f\in C} \err_S(f)$.
\begin{lem}
\label{lem:epw-empirical}
Let $S=(x_1,y_1),\ldots,(x_n,y_n) \in ([N]\times \zo)^n$ be a set of $n$ examples. Then for $M=\epw(T,\eps)$, $$\err_S(M(S)) \leq \opt_S(\Thr_{N,k}) + (k+2)T/n + e^{-\eps T/2} .$$
\end{lem}
\begin{proof}
  As before, we assume that examples in $S$ are in the sorted order. Let $V(S,x)$ denote the value of $v_t$ at the end of running $M(S,x)$. Let $f \in \Thr_{N,k}$ be the interval function with the lowest error on $S$ and let $a_{[k]} = a_1,\ldots,a_k$ be its parameters.

  We first deal with points $x_i$ such that $V(S,x_i) \in \{-T,T\}$. We denote the set of indices of these points by $I$. Observe that if $V(S,x_i) = T$ then $y_i=1$. This is true since for $y_i=0$ the projected walk makes a $-1$ step and then projects to $[-T,T]$. Such step cannot end in $V(S,x_i) = T$. Similarly, if $V(S,x_i) = -T$ then $y_i=0$. This means that  for $i \in I$, $\pr_{M}[M(S,x_i) \neq y_i] = 1/(e^{\eps T/2}+1) \leq e^{-\eps T/2}$.
  Consequently,
  \equ{\sum_{i\in I}\pr_{M}[M(S,x_i) \neq y_i] \leq |I| \cdot e^{-\eps T/2}. \label{eq:err-boundary}}

  We now split the examples into bands according to the endpoints of the step that the projected walk took on them. Specifically, for $v \in \{-T,-T+1,\ldots,T-1\}$ let $I_v$ be the set of indices $i$ such that either $V(S,x_{i-1}) = v$ and $V(S,x_{i}) = v+1$ or $V(S,x_{i-1}) = v+1$ and $V(S,x_{i}) = v$. Let $J \doteq \bigcup_{-T\leq v \leq T-1} I_v$. Note that for $u\neq v$, $I_u \cap I_v = \emptyset$ but $I_{-T}$ and $I_{T-1}$ may include some of the points in $I$. Also the collection of these sets covers all the indices: $I \cup J = [n]$.

  We make several simple observations about examples with indices in $I_v$. Let $i_1 < i_2 <\cdots < i_\ell$ be the indices of points in $I_v$. The labels of points have to alternate, or $y_{i_j} \neq y_{i_{j+1}}$ for all $j\in [\ell-1]$. This is due to the fact that the walk cannot traverse the interval of values $[v,v+1]$ twice in a row in the same direction. We use this to compute the total error of both $M$ and $f$ on points with indices in $I_v$.

  If $y_{i_j} = 1$ then the projected walk made $+1$ step at $i_j$ and therefore $\pr_{M}[M(S,x_{i_j}) \neq y_{i_j}] = 1/(e^{\eps (v+1)/2}+1)$. While if $y_{i_j} = 0$ then the projected walk made $-1$ step at $i_j$ and therefore $\pr_{M}[M(S,x_{i_j}) \neq y_{i_j}] = e^{\eps v/2}/(e^{\eps v/2}+1)$. Note that $$ 1/(e^{\eps (v+1)/2}+1) + e^{\eps v/2}/(e^{\eps v/2}+1) \leq 1 .$$ Therefore for any pair of points with opposite labels the sum of expected errors is at most $1$. The alternation of labels for examples in $I_v$ then implies that
  \equ{\sum_{i\in I_v}\pr_{M}[M(S,x_i) \neq y_i] = \sum_{j\in[\ell]}\pr_{M}[M(S,x_{i_j}) \neq y_{i_j}] \leq \frac{|I_v|+1}{2}, \label{eq:v-err-upper}}
  where the additional $1$ is the bound on the probability of error on a point that has no pair with the opposite label (which happens when the size of $I_v$ is odd).

  Now consider the error of $f=\theta_{a_{[k]}}$ on points in $I_v$. Note that $f$ splits $[N]$ into at most $k+1$ intervals where the value of $f$ is constant. For $r \in [k+1]$ let $J_r$ denote the $r$-th interval and $I_{v,r} \doteq I_v \cap J_r$. The alternation of labels in $I_v$ implies that
   the number of points with indices in $I_{v,r}$ on which $f$ is correct can be larger than the number of points on which $f$ is wrong by at most $1$. That is:
  $$ \sum_{i\in I_{v,r}} |f(x_i)-y_i| \geq \frac{|I_{v,r}|-1}{2} .$$ Hence
  \equ{\sum_{i\in I_v} |f(x_i)-y_i| = \sum_{r\in [k+1]} \sum_{i\in I_{v,r}} |f(x_i)-y_i| \geq \sum_{r\in [k+1]} \frac{|I_{v,r}|-1}{2} = \frac{|I_v|-k-1}{2}. \label{eq:v-err-lower}}

  Combining the inequalities \eqref{eq:v-err-upper} and \eqref{eq:v-err-lower}, we obtain that
  $$ \sum_{i\in I_v}\pr_{M}[M(S,x_i) \neq y_i] \leq \sum_{i\in I_v} |f(x_i)-y_i| +\frac{k+2}{2}.$$ Summing up over all values of $v \in \{-T,-T+1,\ldots,T-1\}$ we get
  \alequn{\sum_{i\in J} \pr_{M}[M(S,x_i) \neq y_i] &= \sum_{v\in \{-T,-T+1,\ldots,T-1\}} \sum_{i\in I_v} \pr_{M}[M(S,x_i) \neq y_i] \\
   & \leq  \sum_{v\in \{-T,-T+1,\ldots,T-1\}} \lp\sum_{i\in I_v} |f(x_i)-y_i| \rp +\frac{k+2}{2}\\
   & = \sum_{i\in J} |f(x_i)-y_i| + \frac{2T(k+2)}{2} \\
   & \leq n \cdot \err_S(f) + T(k+2)}

  Finally, $I \cup J = [n]$ and therefore combining this with equation \eqref{eq:err-boundary} we get:
  \alequn{\err_S(M(S)) &\leq \fr{n} \lp \sum_{i\in J} \pr_{M}[M(S,x_i) \neq y_i] + \sum_{i\in I} \pr_{M}[M(S,x_i) \neq y_i] \rp \\
  & \leq \err_S(f) + \frac{T(k+2)}{n} + \frac{|I|}{n} \cdot e^{-\eps T/2}\\
  & \leq \opt_S(\Thr_{N,k}) + \frac{T(k+2)}{n} + e^{-\eps T/2} .}
\end{proof}

Now by choosing $T = \lceil 2 \ln(2/\alpha)/\eps \rceil$ we can ensure that the empirical error of $\epw(T,\eps)$ is close to the best possible by a function in $\Thr_{N,k}$. To prove that our algorithm generalizes we appeal to generalization properties of differentially private prediction described in Section \ref{sec:generalization}.

\begin{thm}[subsumes Thm.~\ref{thm:thr-intro}]
\label{thm:thr-expectation}
For any $\alpha,\eps > 0$ and $k,N\in \N$, $T = \lceil 2 \ln(2/\alpha)/\eps \rceil$ let $M\doteq \epw(T,\eps)$. Then $M$ is an $\eps$-differentially private prediction algorithm and given $n \geq \frac{4(k+2) \ln(2/\alpha)}{\alpha \eps}$ examples from an arbitrary distribution $\cP$ over $[N]\times \zo$ its output satisfies:
$$\E_{S \sim \cP^n}\lb \err_\cP(M(S)) \rb  \leq e^\eps \cdot (\opt_\cP(\Thr_{N,k}) + \alpha) .$$
In particular, setting $\eps = \alpha/2$ we get that for $n= O(k \ln(1/\alpha)/\alpha^2)$
$$\E_{S \sim \cP^n}\lb \err_\cP(M(S)) \rb  \leq \opt_\cP(\Thr_{N,k}) + \alpha .$$
Further, if $\opt_\cP(\Thr_{N,k})=0$ and $\eps \leq 1/(16\ln(1/\beta))$, then for every $\beta \in (0,1)$,
$$\pr_{S \sim \cP^n}\lb \err_\cP(M(S)) \geq 3\alpha \rb \leq 2\beta.$$
\end{thm}
\begin{proof}
Evaluation of the disagreement error $|M(S,x)-y|$ is $\eps$-differentially private for all $x$ and $y$. Therefore we can apply Lemma \ref{lem:generalize-moment-bound} to obtain:
$$\E_{S \sim \cP^n}\lb \err_\cP(M(S)) \rb  = \E_{S,S' \sim \cP^n}\lb \err_{S'}(M(S)) \rb \leq e^\eps \cdot \E_{S \sim \cP^n}\lb \err_S(M(S))\rb .$$
By applying Lemma \ref{lem:epw-empirical} we get that
$$\E_{S \sim \cP^n}\lb \err_S(M(S))\rb \leq \E_{S \sim \cP^n}\lb \opt_S(\Thr_{N,k})\rb + (k+2)T/n + e^{-\eps T/2} \leq \E_{S \sim \cP^n}\lb \opt_S(\Thr_{N,k})\rb + \alpha .$$
Finally, we recall a well known fact that for any class $C$ and $n$, $\E_{S\sim \cP^n}[\opt_S(C)] \leq \opt_\cP(C)$.
Hence,
$$\E_{S \sim \cP^n}\lb \err_S(M(S))\rb \leq e^\eps \cdot \lp \E_{S \sim \cP^n}\lb \opt_S(\Thr_{N,k})\rb + \alpha \rp \leq e^\eps \cdot (\opt_\cP(\Thr_{N,k}) + \alpha).$$
To establish the high probability bounds for the realizable case, we note that if $\opt_\cP(\Thr_{N,k})=0$ then for every $S$ that includes only the elements in the support of $\cP$, $\opt_S(\Thr_{N,k}) = 0$.
Hence $\err_S(M(S)) \leq \alpha$. Now applying Lemma \ref{lem:generalize-high-prob}, we obtain:
$$\pr_{S,S' \sim \cP^n}\lb \err_{S'}(M(S)) \geq \alpha \cdot e^{2 \sqrt{\eps \ln(1/\beta)}} \rb \leq \beta .$$
Using the condition that $\eps \leq 1/(16\ln(1/\beta))$, we get that $e^{2 \sqrt{\eps \ln(1/\beta)}} \leq e^{1/2} \leq 2$. Hence
$$\pr_{S,S' \sim \cP^n}\lb \err_{S'}(M(S)) \geq  2 \alpha \rb \leq \beta .$$

Now if for some $S$,  $\err_{\cP}(M(S)) \geq  3 \alpha$, then using the fact that $\err_{S'}(M(S))$ is the mean of $n$ independent Bernoulli random variables with bias $\err_{\cP}(M(S))$, we get that with high probability $\err_{S'}(M(S)) \geq  2 \alpha$. Specifically, by Chernoff bound, for $n \geq 6\ln(2)/\alpha$ (which is satisfied by the conditions of our theorem), $$\pr_{S'\sim \cP^n}\lb \err_{S'}(M(S)) \geq  2 \alpha \rb \geq 1/2 .$$ Thus
$$\beta \geq \pr_{S,S'\sim \cP^n}\lb \err_{S'}(M(S)) \geq  2 \alpha \rb \geq \fr{2} \pr_{S \sim \cP^n}\lb \err_{\cP}(M(S)) \geq  3 \alpha \rb.$$
\end{proof}
Our generalization results in Section \ref{sec:generalization} are not strong enough to prove that our algorithm has low expected error with  high probability over $S$ in the agnostic case (while having asymptotically optimal sample complexity). Establishing such a result is an interesting open problem.
\iffull
\section{Stability and Generalization}
\label{sec:generalization}
We now view prediction privacy as a notion of stability and derive generalization properties of private prediction algorithms.
Our results will be stated for a somewhat more general class of algorithms that compute any function of a single data point while satisfying differential privacy.
\begin{defn}[Private evaluation algorithm]
Let $M$ be an algorithm that given a dataset $S\in Z^n$ and a value $z \in Z$ produces a value in a set $W$.
 We say that $M$ is an {$(\eps,\delta)$-differentially private} evaluation algorithm if for every $z\in Z$, the output $M(S,z)$ is  $(\eps,\delta)$-differentially private with respect to $S$.
\end{defn}
In our application the algorithm $M$ will be computing the loss of the prediction produced on $S$ by a prediction algorithm $M'$. Namely, $Z = X\times Y$ and for some loss function $\ell$, $M(S,(x,y)) = \ell(M'(S,x),y)$. Note that by the postprocessing property of differential privacy (\eg \cite{DworkRoth:14}),
this evaluation is differentially private with the same parameters. We state this formally below.
\begin{lem}[Postprocessing] For $Z = X\times Y$ let $M':Z^n\times X \to \R$ be an $(\eps,\delta)$-differentially private prediction algorithm. Then for every loss function $\ell:\R \times Y \to \R$, $M(S,(x,y)) \doteq \ell(M'(S,x),y)$ is an $(\eps,\delta)$-differentially private evaluation algorithm.
\end{lem}

\eat{
We will establish a bound on an arbitrary moment of the empirical average of the function. To do this we will need to observe that the standard group privacy property of differential privacy can be extended to private evaluation algorithms.
\begin{lem}[Group privacy]
\label{lem:group-privacy}
Let $M:Z^n\times Z\to Y$ be an $(\eps,\delta)$-differentially private evaluation algorithm and $k \in \mathbb{N}$.
For all pairs of data sets $S,S'\in Z^n$ differing in at most $k$ elements and all $z\in Z$:
$$D_\infty^{e^{\epsilon(k-1)}\delta}(M(S,z)\|M(S',z)) \leq k \eps\ .$$
\end{lem}
}

We will need the following simple property of $D_\infty^\delta$ to argue about closeness of expectations.
\begin{lem}[e.g.\citep{FeldmanS17}]
\label{lem:dist}
 Let $U$ and $V$ be two random variables over $[0,B]$ such that $D_\infty^\delta(U\|V)\leq \eps$. Then
$\E[U] \leq e^\eps \cdot \E[V] + \delta \cdot B$.
\end{lem}

Let $S= (z_1,z_2,\ldots,z_n)$ and $S'= (z'_1,z'_2,\ldots,z'_n)$ be two sequences of samples drawn randomly and independently from some unknown distribution $\cP$ over $Z$. We consider the relationship between the empirical mean of a differentially private evaluation algorithm and its mean on independently drawn samples, namely between $\cE_S[M(S)] \doteq \frac{1}{n}\sum_{i\in [n]}\E_M [M(S,z_i)]$ and
$\cE_{S'}[M(S)] = \frac{1}{n}\sum_{i\in [n]}\E_M [M(S,z'_i)]$. Clearly, $$\E_{S' \sim \cP^n}[ \cE_{S'}[M(S)]] = \E_{z\sim \cP,\ M}[M(S,z)] = \cE_\cP[M(S)]. $$ Moreover, standard concentration inequalities implies that $\cE_{S'}[M(S)]$ is strongly concentrated around $\cE_\cP[M(S)]$.
Therefore our bounds on $\cE_{S'}[M(S)]$ readily imply bounds on $\cE_{\cP}[M(S)]$ while being easier to state. Note that for $M(S,(x,y)) = \ell(M'(S,x),y)$, $\cE_S[M(S)] = \cE_S[\ell(M'(S))]$ and $\cE_\cP[M(S)] = \cE_\cP[\ell(M'(S))]$, in other words these are exactly the empirical and the expected loss of the predictor given by $M'$. We give the following bound on the $k$-th moment of $\cE_{S'}[M(S)]$.
\begin{lem}
\label{lem:generalize-moment-bound}
Let $M:Z^n\times Z\to [0,B]$ be an $(\eps,\delta)$-differentially private evaluation algorithm and $\cP$ be an arbitrary distribution over $Z$. Then:
$$\E_{S,S' \sim \cP^n} \lb\lp \cE_{S'}[M(S)] \rp^k\rb \leq e^{k^2 \cdot \epsilon} \cdot \E_{S \sim \cP^n}\lb\lp\cE_S[M(S)] + k \delta B\rp^k\rb$$
\end{lem}
\begin{proof}
For a sequence of $k$ indices $I=(i_1,\ldots,i_k) \in [n]^k$ let $S_I$ denote $S$ with every element with index in $I$ replaced with the corresponding element of $S'$. Namely, for each $i$, if exists $j$ such that $i=i_j$, then $z_i$ is replaced with $z'_i$. 
Now, observe that
\equ{\E_{S,S' \sim \cP^n} \lb\lp\frac{1}{n}\sum_{i\in [n]}\E_M [M(S,z'_i)] \rp^k\rb = \frac{1}{n^k} \sum_{I\in [n]^k} \E_{S,S' \sim \cP^n} \lb \prod_{i\in I} \E_M [M(S,z'_{i})] \rb . \label{eq:decomp}}

Using group privacy (Lem.~\ref{lem:group-privacy}) we know that $D_\infty^{ke^{\epsilon(k-1)}\delta}(M(S_I,z'_i)\| M(S,z'_i)) \leq \eps k$. Using Lemma~\ref{lem:dist} we obtain that
$$ \E_M [M(S,z'_{i})] \leq e^{\eps k} \lp \E_M [M(S_I,z'_i)]  +k\delta B\rp.$$
Consequently, $$ \prod_{i\in I} \E_M [M(S,z'_i)] \leq e^{k^2\cdot \eps} \cdot \prod_{i\in I} \lp \E_M [M(S_I,z'_i)] + k\delta B\rp\ .$$
Observe that for $S,S'\sim \cP^n$, $\prod_{i\in I} \lp \E_M [M(S_I,z'_i)] + k\delta B\rp$ is distributed identically to $\prod_{i\in I} \lp \E_M [M(S,z_i)] + k\delta B\rp$. Substituting this into equation \eqref{eq:decomp} we get
\alequn{\E_{S,S' \sim \cP^n} \lb\lp\frac{1}{n}\sum_{i\in [n]}\E_M [M(S,z'_i)] \rp^k\rb & \leq e^{k^2\cdot \eps} \cdot \frac{1}{n^k} \sum_{I\in [n]^k} \E_{S \sim \cP^n} \lb \prod_{i\in I} \lp \E_M [M(S,z_i)] + k\delta B\rp \rb \\
& = e^{k^2\cdot \eps} \cdot \E_{S \sim \cP^n}\lb\lp\frac{1}{n}\sum_{i\in [n]}\E_M [M(S,z_i)] + k\delta B\rp^k\rb.
}
\end{proof}

We now give a simple example of how to obtain high probability generalization bounds from Lemma \ref{lem:generalize-moment-bound}. For simplicity we consider only the case where $\cE_{S}[M(S)]$ is upper bounded by a fixed value $\alpha$ and $\delta = 0$ (such as in our application for  realizable learning in Theorem \ref{thm:thr-expectation}). It can be extended relatively easily to the case then such bound holds with sufficiently high probability and $\delta >0$.
\begin{lem}
\label{lem:generalize-high-prob}
Let $M:Z^n\times Z\to \R^+$ be an $\eps$-differentially private evaluation algorithm and $\cP$ be an arbitrary distribution over $Z$. Assume that for every $S$, $\cE_S[M(S)]\leq \alpha$. Then for every $\beta \in (0,1)$,
$$\pr_{S,S' \sim \cP^n} \lb \cE_{S'}[M(S)] \geq \alpha \cdot e^{2 \sqrt{\eps\ln(1/\beta)}}  \rb \leq \beta .$$
\end{lem}
\begin{proof}
By Markov's inequality and Lemma \ref{lem:generalize-moment-bound}, for every $t \geq 1$
\alequn{ \pr_{S,S' \sim \cP^n} \lb \cE_{S'}[M(S)] \geq t \cdot \alpha \cdot e^{k \eps} \rb &= \pr_{S,S' \sim \cP^n} \lb \lp \cE_{S'}[M(S)]\rp^k \geq \lp t \cdot \alpha \cdot e^{k \eps}\rp^k \rb \\
  &\leq  \frac{\E_{S,S' \sim \cP^n} \lb \lp \cE_{S'}[M(S)]\rp^k\rb}{\lp t \cdot \alpha \cdot e^{k \eps}\rp^k} \\
  &\leq \frac{e^{k^2 \cdot \epsilon} \cdot \E_{S \sim \cP^n}\lb\lp\cE_S[M(S)]\rp^k\rb }{\lp t \cdot \alpha \cdot e^{k \eps}\rp^k} \\
  & \leq  \frac{e^{k^2 \cdot \epsilon} \cdot \alpha^k }{\lp t \cdot \alpha \cdot e^{k \eps}\rp^k} \leq t^{-k} .}

Setting $k = \sqrt{\ln(1/\beta)/\eps}$ and $t = \beta^{-1/k}$, we obtain that $t^{-k} = \beta$ and
$$ t \cdot \alpha \cdot e^{k \eps} = \alpha \cdot e^{\ln(1/\beta)/ \sqrt{\ln(1/\beta)/\eps}} \cdot e^{\eps \cdot \sqrt{\ln(1/\beta)/\eps}} = \alpha \cdot e^{2 \sqrt{\eps\ln(1/\beta)}}. $$
\end{proof}

\remove{
We also prove the following bound on the generalization error of $\eps$-differentially private evaluation algorithms.
\begin{lem}Let $M:Z^n\times Z\to \R^+$ be an $\eps$-differentially private evaluation algorithm and $\cP$ be an arbitrary distribution over $Z$. Then:
$$\E_{S,S' \sim \cP^n} \lb\left(\frac{1}{n}\sum_{i\in [n]}\E_M [M(S,z'_i)] - \frac{1}{n}\sum_{i\in [n]}\E_M [M(S,z_i)]  \right)^k\rb    \leq 2^{k-1} \cdot (e^{k^2\epsilon}-1) \cdot \E_{S \sim \cP^n}\lb\frac{1}{n}\sum_{i\in [n]}\E_M [M(S,z_i)]\rb$$
\end{lem}
\begin{proof}
TBD.
\end{proof}
}

\fi

\section{Discussion}
Several recent works point out risks to the privacy of personal data used to train a predictive model even when the attacker is given only black-box access to the model \citep{ShokriSSS17,CarliniLKES18,LongBWBW18}. In a number of application such access is provided via a prediction interface. While the risks can be mitigated by training the model in a differentially private way \citep{CarliniLKES18}, known theoretical and practical results show that this may substantially reduce the accuracy of the model (\eg \citep{BassilyST14}). In this work we formulated and examined an alternative approach that only aims to ensure that the predictions themselves are differentially private. Further, we focused on understanding of making a single prediction with differential privacy (on an arbitrarily chosen point).

As we have demonstrated, simple privacy-preserving aggregation of labels created by non-private models allows to avoid some of the overheads of training the model differentially privately. Most notably, it removes the dependence on the dimension of the data for some classification and regression problems. Further, we show that algorithms satisfying uniform prediction stability can be used to reduce the overheads of aggregation. Yet, it appears that for many problems, privacy-preserving aggregation leads to a suboptimal algorithm. We therefore ask which other algorithmic approaches can be used to address the problem. Our algorithm for learning thresholds gives an example of an approach that improves on the aggregation-based learning. Finding a more general approach to agnostic learning with prediction privacy that achieves optimal sample complexity is a natural open problem.

Differentially private prediction is also a natural notion of stability. We demonstrate that it leads to (relatively) strong generalization guarantees and exploit these results to analyze our algorithm for learning thresholds. Still, the guarantees we prove are not as strong as those proved for models trained with differentially privacy. An interesting open problem is whether our results can be improved. Specifically, whether the factor $e^{\sqrt{\eps}}$ in Lemma \ref{lem:generalize-high-prob} can be improved to $e^{O(\eps)}$.

A learning algorithm that ensures differential privacy of a single prediction would be suitable for applications in which each user asks few queries and it can be assumed that users do not share their predictions. Finding approaches for dealing with multiple prediction queries (that go beyond the composition properties of differential privacy) is an important direction for further research (in the context of privacy-preserving aggregation this question has been considered in several work described in Section \ref{sec:related}).

\iffull
\printbibliography
\fi

\end{document}